\documentclass[sigconf,screen,nonacm]{acmart}
\usepackage{amsfonts,amsmath}
\usepackage{xeCJK}
\usepackage{booktabs}
\usepackage{makecell}
\usepackage{float}
\usepackage{placeins}
\usepackage{subcaption}
\usepackage{enumerate}
\usepackage{enumitem}
\usepackage{url}
\usepackage{newfloat}
\usepackage{listings}
\usepackage{multirow}
\usepackage{subcaption}
\usepackage{tabularx}
\newcolumntype{C}{>{\centering\arraybackslash}X}
\usepackage[capitalize,noabbrev]{cleveref}
\usepackage[ruled,linesnumbered]{algorithm2e}
\usepackage[breakable,skins]{tcolorbox}
\usepackage{float}
\newtcolorbox{promptbox}[1][]{%
  enhanced,
  colback=blue!3, colframe=blue!25,
  fonttitle=\small\bfseries, title={#1},
  left=4pt, right=4pt, top=4pt, bottom=4pt,
  boxrule=0.5pt,
}



\DeclareCaptionStyle{ruled}{labelfont=normalfont,labelsep=colon,strut=off} 
\floatstyle{ruled}
\newfloat{listing}{tb}{lst}{}
\floatname{listing}{Listing}

\AtBeginDocument{%
  \providecommand\BibTeX{{%
    \normalfont B\kern-0.5em{\scshape i\kern-0.25em b}\kern-0.8em\TeX}}}

\setcopyright{acmlicensed}
\copyrightyear{2026}
\acmYear{2026}
\acmDOI{XXXXXXX.XXXXXXX}
\acmConference[Conference acronym 'XX]{Make sure to enter the correct
  conference title from your rights confirmation email}{June 03--05,
  2018}{Woodstock, NY}
\acmISBN{978-1-4503-XXXX-X/18/06}

\makeatletter
\renewcommand{\@mkauthors@iii}{%
    \gdef\@currentauthors{}%
    \hsize=\textwidth
    \global\setbox\mktitle@bx=\vbox{%
    \unvbox\mktitle@bx
    \centering
    {\normalfont\large\bfseries
      Yuejia Dou\textsuperscript{1*},
      Hesong Wang\textsuperscript{1*},
      Xinyu Zhang\textsuperscript{2},
      Yuchao Ma\textsuperscript{1},
      Tianyu Wang\textsuperscript{2}\\[2pt]
      Zhilin Zhang\textsuperscript{2},
      Chuan Yu\textsuperscript{2},
      Jian Xu\textsuperscript{2},
      Bo Zheng\textsuperscript{2},
      Qi Qi\textsuperscript{1\textdagger}
      \par
    }

    \vspace{4pt}

    {\normalfont\normalsize
      \textsuperscript{1}Gaoling School of Artificial Intelligence,
      Renmin University of China
      \quad
      \textsuperscript{2}Alibaba Group
      \par
    }

    \vspace{4pt}

    {\normalfont\small\ttfamily
  \nolinkurl{douyuejia@ruc.edu.cn}\quad
  \nolinkurl{qi.qi@ruc.edu.cn}
  \par
    }

    \medskip
  }%
}
\makeatother

\begin{document}
\raggedbottom
\title{AIGB-R1: Self-Evolving Generative Auto-Bidding via Hierarchical Planner-Executor Optimization}

\author{Yuejia Dou}
\authornote{Equal contribution}
\email{douyuejia@ruc.edu.cn}
\affiliation{%
  \institution{\mbox{Gaoling School of Artificial Intelligence,} Renmin University of China}
    \city{Beijing}
  \country{China}}

\author{Hesong Wang}
\authornotemark[1]
\email{whs040114@gmail.com}
\affiliation{%
  \institution{\mbox{Gaoling School of Artificial Intelligence,} Renmin University of China}
    \city{Beijing}
  \country{China}}

\author{Xinyu Zhang}
\email{zxy479640@alibaba-inc.com}
\affiliation{%
  \institution{Alibaba Group}
    \city{Beijing}
  \country{China}}

\author{Yuchao Ma}
\email{yc.ma@ruc.edu.cn}
\affiliation{%
  \institution{\mbox{Gaoling School of Artificial Intelligence,} Renmin University of China}
    \city{Beijing}
  \country{China}}

\author{Tianyu Wang}
\email{yves.wty@alibaba-inc.com}
\affiliation{%
  \institution{Alibaba Group}
    \city{Beijing}
  \country{China}}

\author{Zhilin Zhang}
\email{zhangzhilin.pt@alibaba-inc.com}
\affiliation{%
  \institution{Alibaba Group}
    \city{Beijing}
  \country{China}}

\author{Chuan Yu}
\email{yuchuan.yc@alibaba-inc.com}
\affiliation{%
  \institution{Alibaba Group}
    \city{Beijing}
  \country{China}}
  
\author{Jian Xu}
\email{xiyu.xj@alibaba-inc.com}
\affiliation{%
  \institution{Alibaba Group}
    \city{Beijing}
  \country{China}} 
  
\author{Bo Zheng}
\email{bozheng@alibaba-inc.com}
\affiliation{%
  \institution{Alibaba Group}
    \city{Beijing}
  \country{China}} 
  
\author{Qi Qi}
\email{qi.qi@ruc.edu.cn}
\authornote{Corresponding author}
\affiliation{%
  \institution{\mbox{Gaoling School of Artificial Intelligence,} Renmin University of China}
    \city{Beijing}
  \country{China}}
  


\begin{abstract}
Auto-bidding plays an essential role in online advertising, automatically adjusting bids for advertisers to optimize their commercial goals. The emerging AI-Generated Bidding (AIGB) paradigm widely adopts generative modeling to optimize bidding strategies, yet suffers from the limited mode coverage of offline datasets and inadequate task-state understanding, hindering effective exploration of optimal strategies. Large Language Models (LLMs), with prior world knowledge and reasoning capabilities, offer a promising approach to overcome these limitations. However, directly applying LLMs to auto-bidding tasks faces inherent challenges in limited numerical precision, hallucinations, and inference latency. To address these limitations, we propose \textbf{AIGB-R1}, a hierarchical self-evolving auto-bidding framework aiming to enhance \textbf{AI}-\textbf{G}enerated \textbf{B}idding via LLMs' \textbf{R}easoning capabilities, comprising a high-level Planner module for macro-level strategy planning and a low-level Executor module for fine-grained decision-making.
Building upon this, we design an experience-driven self-evolving loop, enabling autonomous strategy exploration and optimization from accumulated experience. We adopt a two-stage pipeline of offline pre-training and post-training alignment, and build an interactive bidding simulation environment for strategy rollout. Furthermore, we propose Decoupled Group Relative Policy Optimization (D-GRPO) to achieve end-to-end optimization via advantage decoupling. 
Experimental results on a large-scale public dataset demonstrate the effectiveness of AIGB-R1.
\end{abstract}




\maketitle

\section{Introduction}
\label{sec:intro}

With the rapid digitalization of commerce, online advertising platforms have continually expanded their traffic coverage, becoming critical channels for businesses to attract target audiences and boost sales \cite{evans2009online,ha2008online,wang2015real}. Given billions of impressions and intense competition, the traditional paradigm of manual bid adjustment has become impractical. To alleviate this burden, major advertising platforms provide auto-bidding services that automatically determine bids for each impression to optimize advertisers' commercial goals \cite{balseiro2021robust,deng2021towards,ou2023deep}.
In the emerging AI-Generated Bidding (AIGB) paradigm, auto-bidding can be formulated as a sequential decision-making task, where the objective is to generate optimal bidding parameters that maximize cumulative conversions subject to constraints such as cost-per-action (CPA) and budget, based on historical bidding sequences and current market states \cite{guo2024generative,he2021unified}. Within this paradigm, generative auto-bidding methods represented by Decision Transformer (DT) \cite{chen2021decision,zheng2022online} have attracted widespread attention \cite{gao2025generative,li2025gas,dou2026gam}. These methods employ generative models to learn optimal strategies from historical bidding trajectories collected from online advertising platforms, autoregressively generating bidding parameters conditioned on return-to-go, state, and action sequences.

However, existing generative auto-bidding methods face critical challenges. First, their performance is restricted by the quality and coverage of the offline dataset, leading to degraded decision quality when encountering out-of-distribution states \cite{li2026lbm,li2025gas}. Second, existing generative auto-bidding methods simply condition policies on return-to-go without fully understanding the true value of actions and specific task states, making it difficult to plan optimal strategies \cite{ajay2022conditional,gao2025generative}. Third, existing generative methods are typically trained under fixed preference distributions and struggle to adapt to diverse advertiser preferences in real-world bidding services, while retraining for different preferences incurs substantial training costs \cite{huang2025generative,song2026dara,dou2026gam}.
These challenges limit the further development and widespread adoption of generative auto-bidding services \cite{wang2017display}, but the rapid advancement of Large Language Models (LLMs) in recent years offers new insights into how to overcome these bottlenecks \cite{bai2023qwen,yang2025qwen3,guo2025deepseek}. With prior world knowledge and reasoning capabilities, LLMs are expected to transcend the limitations of offline dataset coverage and adapt to diverse bidding strategy preferences through semantic understanding and multi-task generalization.
However, directly applying LLMs to auto-bidding tasks faces inherent limitations. First, auto-bidding requires precise control over continuous numerical action spaces when generating bidding parameters, yet LLMs suffer from inherent limitations in numerical decision-making, including insufficient precision and hallucinations \cite{forootani2025survey,huang2025survey,lv2026decisionllm}. 
Moreover, online advertising platforms require auto-bidding agents to process numerous bid requests under tight latency constraints, while LLMs face a fundamental trade-off between decision-making quality and inference latency: larger LLMs possess strong decision-making and planning capabilities but struggle to meet the low-latency requirements, whereas significantly reducing model size alleviates latency issues but inevitably leads to substantial degradation in decision-making and planning capabilities
\cite{cai2025rtbagent,huang2025survey,xiao2025densing}.

Several works have explored integrating LLM techniques into auto-bidding systems from different perspectives. 
One straightforward approach uses an LLM as the decision-maker, leveraging reasoning capabilities with memory and reflection mechanisms to output bidding decisions, but is directly constrained by the aforementioned numerical precision, hallucination, and latency issues \cite{cai2025rtbagent}.
Other methods optimize generative policy models through semantic representation embedding, or through post-training policy search techniques, but fundamentally fail
to fully exploit the decision-making and planning potential of LLMs \cite{zhu2026role,dou2026gam,li2025gas}. Recently, LBM \cite{li2026lbm} designed a hierarchical framework with separate modules dedicated to high-level reasoning and low-level decision-making respectively, effectively retaining LLM reasoning and planning capabilities while circumventing inherent bottlenecks, representing a promising architectural direction for LLM-empowered auto-bidding \cite{wang2023hibid,song2026dara,wan2025think,ahn2022can}.

However, the above methods still face the following challenges when learning optimal auto-bidding strategies. 
First, existing methods typically rely on offline pre-trained value functions to evaluate strategy quality, but these value functions often suffer from extrapolation 
errors on out-of-distribution strategies, struggling to guide exploration toward potentially optimal regions \cite{fujimoto2019off,kumar2020conservative,
kumar2019stabilizing}.
Second, existing methods' post-training search and alignment typically rely on repeated random sampling, which has limited exploration efficiency in auto-bidding tasks with
high-dimensional continuous strategy spaces, generating massive redundant trajectories while failing to effectively reuse historical rollout experience for self-improvement
\cite{guo2024generative,gao2025generative,li2024trajectory}.
Finally, existing methods either perform training-free test-time search or only optimize the high-level reasoning module while freezing the low-level execution module, thereby limiting convergence to globally optimal strategies compared with end-to-end optimization.

To overcome these challenges, we propose \textbf{AIGB-R1}, a hierarchical self-evolving auto-bidding framework designed to enhance \textbf{AI}-\textbf{G}enerated \textbf{B}idding via LLMs' \textbf{R}easoning capabilities, achieving autonomous exploration and evolution of bidding strategies through an experience-driven self-evolving loop.
AIGB-R1 temporally decouples the auto-bidding task into two stages: macro-level strategy planning and fine-grained bidding decision-making, based on which we construct a high-level Planner module and a low-level Executor module, respectively. 
Specifically, the Planner module employs a large LLM to reason about macro-level strategies for each bidding period and generate a structured semantic strategy prompt that remains fixed and is reused throughout the period. The Executor module employs Prompt Decision Transformer (PDT) to generate bidding parameters at each timestep within the bidding period, conditioned on both observed numerical sequences and the Planner's strategy prompts. Building upon this, we introduce an experience-driven self-evolving loop that enables AIGB-R1 to autonomously explore and optimize strategies from accumulated rollout experience, rather than blindly exploring through random sampling.
We adopt a two-stage training pipeline: In the pre-training stage, we train the Executor to equip it with the capability to fuse observed sequences and strategy prompts for bidding decisions. In the post-training stage, we build an interactive bidding simulation environment for strategy rollout, mitigating extrapolation errors associated with offline value estimation. We design Decoupled Group Relative Policy Optimization (D-GRPO) to support inter-layer credit assignment, enabling end-to-end optimization of the Planner and Executor. 
In summary, our contributions are:
\begin{itemize}[leftmargin=*,noitemsep,topsep=0pt]
    \item We propose AIGB-R1, a hierarchical self-evolving auto-bidding framework comprising a Planner for macro-level strategy planning and an Executor for fine-grained bidding decision-making, fully leveraging LLMs' reasoning and planning capabilities while effectively circumventing their inherent limitations in limited numerical precision, hallucinations and inference latency.
     \item We design an experience-driven self-evolving loop that enables AIGB-R1 to autonomously explore and optimize strategies from accumulated rollout experience, continuously approaching optimal bidding strategies through sustained interaction with the simulation environment during post-training.
    \item We build an interactive bidding simulation environment for strategy rollout and propose D-GRPO to support inter-layer credit assignment through advantage decoupling, enabling end-to-end optimization of the LLM Planner and PDT Executor.
    \item Experimental results on a large-scale public dataset demonstrate that AIGB-R1 significantly outperforms various state-of-the-art baselines, validating the effectiveness of the proposed framework.
\end{itemize}
\section{Related Work}
\label{sec:related work}
\subsection{Auto-bidding Methods}
Auto-bidding helps advertisers achieve their commercial goals by adaptively optimizing bids for each impression \cite{he2021unified,xu2024auto}. Early methods primarily relied on predefined rules or online learning to adjust strategies \cite{chen2011real,yu2017online}. As online bidding environments grew increasingly complex, reinforcement learning methods such as USCB \cite{he2021unified}, SORL \cite{mou2022sustainable}, and MAAB \cite{wen2022cooperative} became essential for auto-bidding. 
Due to the risks of online interaction, offline RL methods including BCQ \cite{fujimoto2019off}, CQL \cite{kumar2020conservative}, and IQL \cite{kostrikov2021offline} have gained widespread attention for their ability to learn effective strategies from offline datasets. However, these methods are subject to the Markov assumption, which limits their sequential modeling capability. 

Recently, generative models represented by Decision Transformer (DT) \cite{chen2021decision,zheng2022online} have opened a new paradigm for auto-bidding. DiffBid \cite{guo2024generative} uses diffusion models to generate bidding trajectories, whereas CBD \cite{li2025generative} introduces a diffusion completer-aligner for trajectory completion. GAS \cite{li2025gas} performs post-training search via MCTS; GAVE \cite{gao2025generative} employs value-guided exploration. PRO-Bid \cite{wu2026constraint} introduces Pareto-prioritized regret optimization, while GRAD \cite{lei2026generative} develops a large-scale pre-trained DT-based bidding model. For the broader auto-marketing setting, GAM \cite{dou2026gam} employs DTs for joint bidding and coupon distribution with GRPO-based preference alignment. 
Several works have begun exploring LLM-empowered auto-bidding. RTBAgent \cite{cai2025rtbagent} uses prompt engineering to directly output bidding decisions via LLMs, SemBid \cite{zhu2026role} integrates LLM semantic representations to assist generative sequence modeling, LBM \cite{li2026lbm} designs a hierarchical Think-Act architecture and employs GQPO for fine-tuning.
\subsection{LLM for Decision-Making}
Recently, the autonomous decision-making potential of LLMs has been extensively explored. Early works such as ReAct \cite{yao2022react} and Reflexion \cite{shinn2023reflexion} established the foundational paradigm of reasoning, acting, and reflection. However, recent studies reveal inherent limitations of LLMs as direct decision-makers \cite{schmied2025llms}, motivating researchers to leverage LLMs for macro-level planning rather than direct action execution. SayCan \cite{ahn2022can}, LgTS \cite{shukla2023lgts}, and DART-LLM \cite{wang2024dart} translate LLM-generated language plans into executable actions through feasibility assessment and effectiveness scoring, while PAR \cite{he2024words} and ACE \cite{wan2025think} propose hierarchical collaboration frameworks between LLMs and RL agents.

Meanwhile, reinforcement learning has been widely adopted to optimize the decision-making capabilities of LLMs. RLHF \cite{ouyang2022training} optimizes models using human feedback to construct reward signals. DPO \cite{rafailov2023direct} and its extensions \cite{azar2024general,ethayarajh2024kto,meng2024simpo,xu2024contrastive} directly fine-tune models with human preference datasets. GRPO \cite{shao2024deepseekmath} improves the stability and efficiency of fine-tuning through group relative advantage estimation. GiGPO \cite{feng2026group} introduces fine-grained credit assignment for multi-step decision-making tasks.
\subsection{Self-Evolving Agents}


Self-evolving agents iteratively improve policies or system components through experience and feedback \cite{gao2025survey,fang2025comprehensive,tao2024survey}. This research covers both training-time self-evolution through self-generated data or environment interaction and test-time improvement. Prior work has explored experience reuse to improve later decisions. Reflexion \cite{shinn2023reflexion}, ExpeL \cite{zhao2024expel}, and ReasoningBank \cite{ouyang2026reasoningbank} retain interaction experience as verbal reflections or memories to iteratively improve decision quality. Other studies improve agent policies through parameter updates. Re-ReST \cite{dou2025rerestreflectionreinforcedselftraininglanguage} refines self-generated data using environment feedback and reflection, while RAGEN \cite{wang2025ragen} and WebRL \cite{qi2025webrl} iteratively optimize policies through multi-turn environment interaction and reinforcement learning.
Experience reuse and parameter learning can also reinforce each other. EvolveR \cite{wu2025evolver} distills interaction trajectories into reusable strategies and combines experience retrieval with RL policy updates, while SkillRL \cite{xia2026skillrl} and ARISE \cite{li2026ariseagentreasoningintrinsic} explore the joint improvement of skill libraries and policies. AIGB-R1 focuses on training-time self-evolution, using accumulated experience and feedback to guide autonomous strategy exploration and D-GRPO for joint hierarchical optimization.


\section{Preliminary}
\label{sec:pre}
\subsection{Problem Formulation}
We consider the auto-bidding problem with cost-related constraints. Within a bidding period, we assume that $I$ impression opportunities arrive sequentially, indexed by $i$. Advertisers submit bids on the platform to compete for each impression. For each arriving impression $i \in [I]$, an advertiser wins when its bid $b_i$ exceeds those of other advertisers, and pays the winning cost $c_i$, which is typically the highest bid of the other advertisers in a second-price auction setting. The advertiser's objective is to maximize the cumulative value of won impressions $\sum_{i}
o_i v_i$ over the entire bidding period, where $v_i$ denotes the value of the impression and $o_i$ indicates whether the advertiser wins it.
Moreover, to control advertising delivery performance, advertisers typically need to satisfy the budget and multiple Key Performance Indicator (KPI) constraints \cite{he2021unified}. The budget constraint is expressed as $\sum_{i} o_i c_i \leq B$, where $c_i$ is the cost of impression $i$ and $B$ is the budget. Other KPI constraints are more complex and can be uniformly expressed as:
\begin{equation}
  \frac{\sum_{i} c_{ij} o_i}{\sum_{i} p_{ij} o_i} \leq C_j,
  \label{eq:kpi}
\end{equation}
where $C_j$ is the upper bound of the $j$-th constraint, $p_{ij}$ can be any performance indicator such as return or constant, and $c_{ij}$ is the cost of the $j$-th constraint.
Therefore, given $J$ constraints, the auto-bidding problem can be formulated as:
  \begin{align}
  \max \quad & \sum_{i} o_i v_i \notag \\
  \text{s.t.} \quad & \sum_{i} o_i c_i \leq B, \notag \\
  & \frac{\sum_{i} c_{ij} o_i}{\sum_{i} p_{ij} o_i} \leq C_j, \quad \forall\, j,
  \label{eq:mcb} \\
  & o_i \in \{0, 1\}, \quad \forall\, i. \notag
  \end{align}

A previous study \cite{he2021unified} has shown that the optimal solution is:
  \begin{equation}
  b_i^* = \lambda_0 v_i + \sum_{j=1}^{J} \lambda_j p_{ij} C_j,
  \label{eq:optimal_bid}
  \end{equation}
where $b_i^*$ is the optimal bid for impression $i$, and $\lambda_j, j \in \{0, \ldots, J\}$ are the optimal bidding parameters. However, the uncertainty and dynamics of online advertising systems make it impractical for advertisers to directly calculate these parameters. Instead, these parameters should be regularly adjusted in response to the dynamic environment, making auto-bidding a sequential decision-making task of iteratively identifying the optimal bidding parameters.
\subsection{Sequential Modeling of Auto-Bidding}
Under the sequential modeling formulation of auto-bidding, we first divide the bidding period into $T$ discrete timesteps. At each timestep
$t$, the auto-bidding agent receives a real-time advertising state $s_t \in \mathcal{S}$ and then outputs an action $a_t \in \mathcal{A}$
according to its policy $\pi$ to adjust the bidding parameters. The state transition dynamics of the advertising environment are unknown, and
the next state $s_{t+1} \in \mathcal{S}$ depends not only on the current state $s_t$ and action $a_t$, but also potentially on the historical
bidding trajectory $\tau$. After transitioning to the next state, the environment returns a reward $r_t$, representing the cumulative
impression value obtained during timestep $t$. The auto-bidding agent's objective is to maximize the cumulative impression value $\sum_{t}
r_t$ over the entire bidding period.
The detailed description of our modeling is as follows:

\begin{itemize}[leftmargin=*, noitemsep]
  \item \textbf{State $s_t$:} The state $s_t$ describes the advertising status at timestep $t$, which can include the remaining time of the bidding period, remaining budget, budget consumption speed, and other historical statistics.
  \item \textbf{Action $a_t$:} The action $a_t$ represents the adjustment of bidding parameters $\lambda_j$ ($j=0,\ldots,J$) at timestep $t$, modeled as $(a_t^{\lambda_0}, \ldots, a_t^{\lambda_J})$.
  \item \textbf{Reward $r_t$:} The reward $r_t$ represents the value contributed to the objective during the period from timestep $t$ to $t+1$.
  \item \textbf{Return-To-Go (RTG) $R_t$:} The RTG value $R_t$ represents the sum of rewards to be obtained in the future timesteps:
  \begin{equation}
    R_t = \sum_{t'=t}^{T} r_{t'},
    \label{eq:rtg}
  \end{equation}
  where $T$ is the final timestep.
\end{itemize}

\section{Method}
\label{sec:method}
\subsection{Framework Overview}
\begin{figure*}[t]
    \captionsetup{skip=4pt} 
    \centering    
    \includegraphics[width=0.92\linewidth]{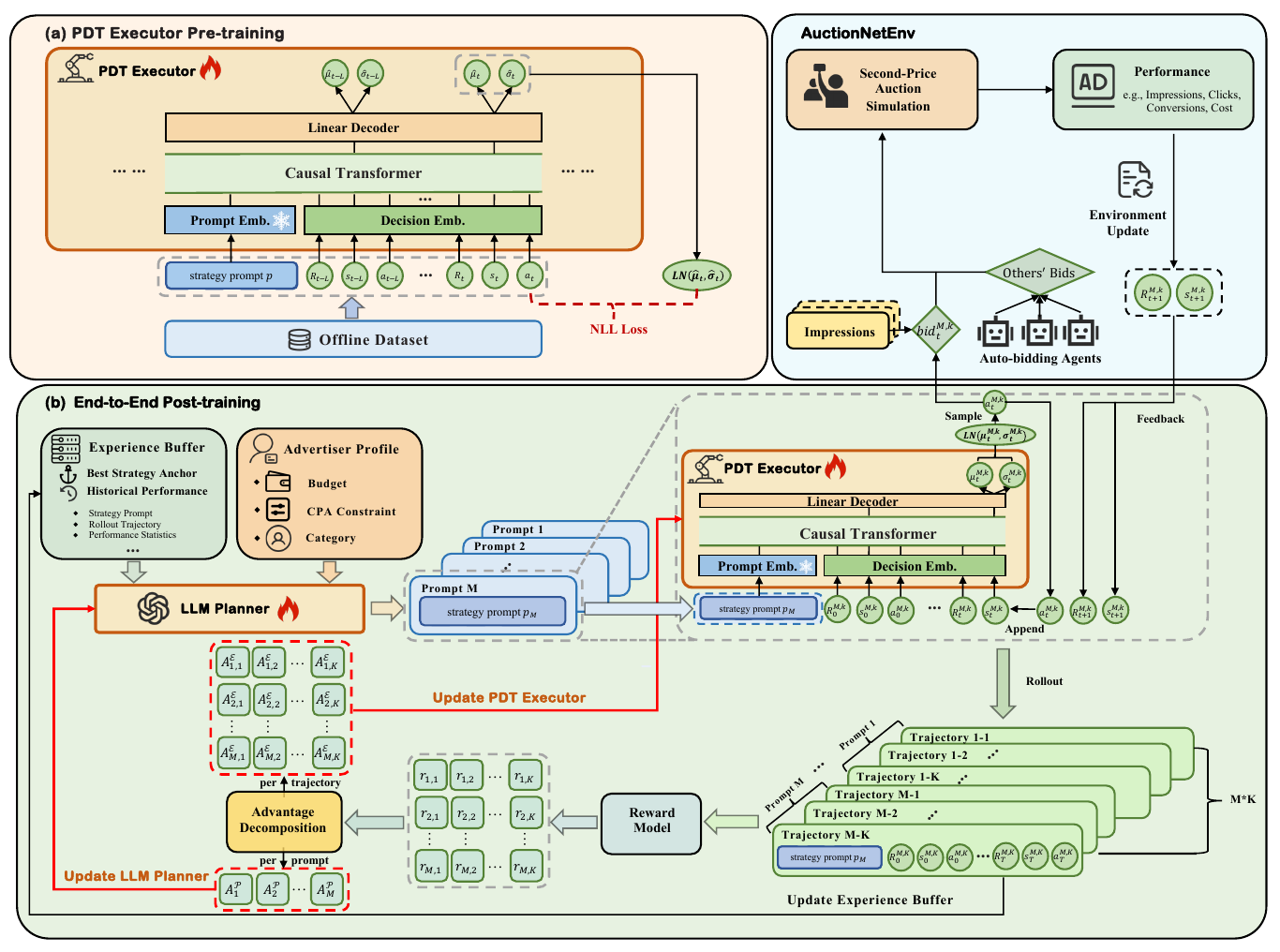}
    \caption{Overall framework of AIGB-R1.}
    \Description{Diagram showing the overall framework of AIGB-R1.}
    \label{fig:framework}
\end{figure*}
In the emerging AIGB paradigm, generative models exemplified by Decision Transformer (DT) have been widely adopted for their ability to learn from offline datasets. However, these methods are limited by the mode coverage of offline datasets and struggle to fully understand task states and strategy preferences, motivating us to explore a new approach that leverages the semantic understanding, reasoning, and planning capabilities of LLMs to empower auto-bidding.
However, LLMs' inherent limitations, including limited numerical precision, hallucinations, and inference latency, make it impractical to directly use a single LLM for end-to-end bidding. 

To address these limitations, a practical architectural design is to decouple strategy reasoning and real-time decision-making into two separate modules using a hierarchical framework, where a large LLM asynchronously performs macro-level
strategy reasoning, while a lightweight decision model makes real-time decisions at each timestep based on the derived strategy and current states.

Accordingly, we temporally decompose the auto-bidding task into macro-level strategy planning followed by fine-grained bidding decision-making, handled by a high-level Planner and a low-level Executor respectively. Specifically, the Planner module employs a large LLM to reason about macro-level strategies at the beginning of each bidding period and generates a structured strategy prompt that remains fixed and is reused throughout the period. The Executor module employs Prompt Decision Transformer (PDT) as a
lightweight decision-maker, generating bidding parameters at each timestep within the bidding period, conditioned on both the observed numerical sequences and the strategy prompt, without invoking the Planner for every bid request.

Next, we formally describe the sequential decision-making process of AIGB-R1. Consider a bidding period for advertiser $g$. Before the bidding period begins, we obtain the advertiser's contextual information $\mathcal{C}_g$ (including budget, CPA constraints, and advertising category) and its historical bidding performance from previous periods $\mathcal{H}_g$. The high-level LLM Planner $\pi_{\theta_{\mathcal{P}}}$ then plans the macro-level strategy (e.g., risk preference) for the current bidding period based on $\mathcal{C}_g$ and $\mathcal{H}_g$, generating a structured strategy prompt $p$ that contains enumerated bidding strategy labels and macro-level strategy reasoning:
\begin{equation}                                                                         
    p \sim {\pi_{\theta_{\mathcal{P}}}}\bigl(\cdot \mid\mathcal{C}_g, \mathcal{H}_g\bigr), \quad                                                                                    
    \label{eq:planner}                                                                     
\end{equation}
Notably, this planning step is completed before the bidding period begins, and the generated strategy prompt remains fixed throughout the entire period.
Subsequently, at each timestep $t$ within the bidding period, the low-level PDT Executor $\pi_{\theta_{\mathcal{E}}}$ generates an action $a_t$ conditioned on the observed numerical sequences of RTG values, states, and actions $(R_{t-L:t}, s_{t-L:t}, a_{t-L:t-1})$ along with the strategy prompt $p$, where $L$ denotes the context length of PDT:
\begin{equation}
  a_t \sim \pi_{\theta_{\mathcal{E}}}\bigl(\cdot \mid p, R_{t-L:t}, s_{t-L:t}, a_{t-L:t-1}\bigr)
  \label{eq:executor}
\end{equation}
PDT embeds the strategy prompt as a conditional prefix to the standard Decision Transformer's input sequence, enabling the Planner's macro-level strategy to naturally influence the conditional autoregressive generation of actions at each timestep.

The overall framework of AIGB-R1 is illustrated in Figure~\ref{fig:framework}. We adopt a two-stage training pipeline of pre-training and post-training alignment: In the pre-training stage, we pre-train the Executor offline to equip it with the capability to fuse observed sequences and strategy prompts for bidding decisions. In the post-training stage, we fine-tune the entire framework end-to-end via reinforcement learning to further discover optimal strategies.
\subsection{PDT Executor Pre-training with Strategy Prompts}
Despite LLMs' strong semantic understanding and reasoning capabilities, directly using them as the low-level Executor faces fundamental limitations. First, industrial bidding systems typically require millisecond-level response times, yet LLMs struggle to simultaneously achieve high performance and low latency. Second, LLMs are inherently unsuitable for tasks with continuous numerical action spaces such as auto-bidding, where the decision model must output precise bidding parameters at each timestep, while LLMs' encoding of numerical inputs as discrete tokens inevitably introduces precision loss and hallucination risks. Furthermore, current public LLMs have not been specifically pre-trained for auto-bidding tasks, meaning their general prior knowledge can hardly provide effective inductive biases for predicting bidding parameters, making it more practical to directly train a dedicated lightweight decision model on offline bidding trajectories.

These limitations motivate us to adopt PDT rather than an LLM as the core architecture of the low-level Executor. PDT extends the standard DT by introducing a textual strategy prompt as a conditional sequence prefix.
Specifically, PDT embeds the strategy prompt $p$ into vector representations and prepends them to the standard DT input sequence $(R_{t-L:t},s_{t-L:t}, a_{t-L:t-1})$, forming the extended input sequence $(p, R_{t-L:t}, s_{t-L:t}, a_{t-L:t-1})$. At timestep $t$, PDT autoregressively generates action $a_t$ conditioned on this sequence. In this way, the strategy prompt influences the conditional autoregressive sampling of actions at each subsequent timestep through the self-attention mechanism, achieving strategy-conditioned decision-making.

The strategy prompt $p$ of the PDT Executor contains two forms of information: enumerated strategy labels and macro-level strategy reasoning. The strategy labels cover predefined discrete strategy dimensions such as risk preference and budget pacing, while the macro-level strategy reasoning provides a summary of historical bidding performance along with semantic reasoning and detailed explanations for the current period's bidding strategy. This dual design balances structural and semantic expressiveness: enumerated strategy labels effectively reduce the variance of strategy learning, enabling the Executor to efficiently learn coarse-grained strategy behavior patterns, while the
macro-level strategy reasoning provides fine-grained semantic strategy expressions, compensating for the strategy details and decision rationale that discrete labels cannot capture. We use pre-trained text embedding layers to process strategy labels and macro-level strategy reasoning into vector representations of equal size, and concatenate them to form the complete strategy prompt representation $p$.

During the pre-training stage of the PDT Executor, we perform offline supervised training on PDT. Whereas previous DT methods learn deterministic policies $\pi_\theta(a_t|R_{t-L:t}, s_{t-L:t}, a_{t-L:t-1})$, we opt to learn a stochastic policy that maximizes the dataset likelihood, as stochastic policies enable efficient strategy exploration through repeated sampling during post-training and can be directly fine-tuned with policy-based reinforcement learning algorithms.
Specifically, since the action $a_t$ is a non-negative continuous value, we adopt the log-normal distribution to model the stochastic policy:
\begin{equation}
    \small
    \begin{split}
    &\pi_{\theta_{\mathcal{E}}}(a_t \mid p, R_{t-L:t}, s_{t-L:t}, a_{t-L:t-1}) = \\
    &\mathcal{LN}\bigl(\mu_{\theta_{\mathcal{E}}}(a_t \mid p, R_{t-L:t}, s_{t-L:t}, a_{t-L:t-1}),
  \sigma_{\theta_{\mathcal{E}}}(a_t \mid p, R_{t-L:t}, s_{t-L:t}, a_{t-L:t-1})\bigr)
    \end{split}
    \label{eq:pdt_policy}
\end{equation}
During pre-training, we train the PDT Executor to fit the offline data distribution by maximizing the log-likelihood of the stochastic policy over the trajectories in the training dataset, equivalently minimizing the Negative Log-Likelihood (NLL) loss:
\begin{equation}
  \small
  \mathcal{L}_{\text{NLL}} = \frac{1}{L+1}\mathbb{E}_{(p,R,s,a)\sim \mathcal{D}} \left[
  - \sum_{t'=t-L}^{t} \log \pi_{\theta_{\mathcal{E}}}\bigl(a_{t'} \mid p, R_{t'-L:t'}, s_{t'-L:t'},
  a_{t'-L:t'-1}\bigr) \right]
  \label{eq:nll_loss}
  \end{equation}

\subsection{Self-Evolving Post-training via D-GRPO}
Limited by the quality and coverage of offline training datasets, pre-trained generative models often converge to local optima. We therefore employ reinforcement learning in the post-training stage to fine-tune the Planner and Executor end-to-end for further performance improvement. However, conventional RL methods typically rely on repeated random trial-and-error, resulting in massive redundant trajectories and slow convergence. To this end, we build a self-evolving post-training framework consisting of the interactive bidding simulation environment AuctionNetEnv, an experience-guided hierarchical rollout mechanism, and the D-GRPO algorithm for inter-layer credit assignment, thereby enabling AIGB-R1 to internalize and leverage accumulated experience and transforming strategy optimization from undirected trial-and-error into a more effective self-improvement process.

\subsubsection{Interactive Bidding Simulation Environment}
To enable the auto-bidding agent to obtain reward signals through rollouts for guiding strategy optimization, we build a flexible interactive bidding simulation environment, AuctionNetEnv, based on a public large-scale auto-bidding decision-making benchmark \cite{su2024auctionnet}. AuctionNetEnv enables the auto-bidding agent to compete with diverse opponent bidding agents, with these opponents trained using different decision-making algorithms (PID, Online LP, CQL, IQL, etc.).
We provide the construction details of AuctionNetEnv and validate its distributional consistency and fidelity in Appendix \ref{app:AuctionNetEnv}.

At each training step, the auto-bidding agent samples a bidding task $g$ from AuctionNetEnv's task pool $\mathcal{G}$, described by the advertiser's contextual information $\mathcal{C}_g$ including budget, CPA constraints, and advertising category, and then executes a complete bidding period for this task. At each timestep $t$, the agent generates an action $a_t$ and computes the bid $b_t$. The environment then conducts the auction with the agent's bid $b_t$ and the opponents' bids, returns outcomes such as impressions, conversions $d_t$, and cost $c_t$ for this timestep, and updates the next state $s_{t+1}$.
At the end of the bidding period, the reward model of AuctionNetEnv returns a trajectory reward $r$ based on the agent's cumulative performance under the CPA constraint $C$, incorporating the constraint satisfaction objective into the reward signal to encourage a balance between maximizing impression value and satisfying constraints:
\begin{equation}
\label{eq:reward}
    \begin{cases}
    CPA = \frac{\sum_t c_t}{\sum_t d_t}\\
    \mathbb{P}(CPA;C)=\min \left\{ \left( \dfrac{C}{CPA} \right)^{\eta}, 1 \right\} \\
    r = \mathbb{P}(CPA;C)\cdot \sum_t d_t
    \end{cases}
\end{equation}
where $\eta$ is a hyperparameter that controls the penalty intensity for constraint violation.

By conducting online agent-environment interaction through AuctionNetEnv, AIGB-R1 mitigates the out-of-distribution extrapolation errors in offline value estimation methods, thereby providing more reliable learning signals for strategy optimization.

\subsubsection{Experience-Guided Hierarchical Rollout}
To leverage historical experience for transforming strategy optimization from random trial-and-error into a more effective self-improvement process, we design an experience-guided hierarchical rollout mechanism that enables the LLM Planner to explore optimal strategies based on accumulated experience during post-training. We observe that optimal strategies for different bidding tasks often differ significantly due to variations in budget scale and CPA constraint tightness, necessitating task-specific experience accumulation. Therefore, we maintain an experience buffer $\mathcal{H}_g$ for each bidding task $g$ in AuctionNetEnv, consisting of two components: the best-strategy anchor and the sliding memory window. The best-strategy anchor records the strategy prompt and bidding trajectory summary associated with the highest average reward across all rollouts for task $g$, serving as a reference direction for subsequent exploration and enabling the Planner to refine its search within the neighborhood of the known best strategy. The sliding memory window retains detailed bidding performance from recent periods in a FIFO manner, including the strategy prompt, bidding
trajectory, trajectory reward, and corresponding performance statistics of each bidding period, allowing the Planner to observe performance differences across strategies on task $g$ for targeted adjustments.

Credit assignment is a core challenge in end-to-end RL training of hierarchical frameworks like AIGB-R1, as it involves assessing the respective contributions of the Planner's strategy selection and the Executor's bidding execution to overall performance. 
We address this challenge through hierarchical rollout. Specifically, for each bidding period of task $g$, we first use the LLM Planner to reason about potentially optimal strategies by referencing the advertiser's contextual information $\mathcal{C}_g$ and historical bidding experience $\mathcal{H}_g$, and output the $M$ most promising and diverse strategy prompts $\{p_1, p_2, \dots, p_M\}$:
\begin{equation}
  p_m \sim \pi_{\theta_{\mathcal{P}}}\bigl(\cdot \mid \mathcal{C}_g, \mathcal{H}_g\bigr), \quad m
   = 1, 2, \dots, M,
  \label{eq:planner_rollout}
\end{equation}
Then, for each strategy prompt $p_m$, we invoke the Executor to run $K$ rollouts in AuctionNetEnv, producing $K$ bidding trajectories
\begin{equation}
  \tau_{m,k} \sim \pi_{\theta_{\mathcal{E}}}\bigl(\cdot \mid p_m\bigr), \quad k = 1, 2, \dots, K,
  \label{eq:executor_rollout}
\end{equation}
After all rollouts, the strategy prompts and bidding trajectories from this period are appended to the experience buffer $\mathcal{H}_g$.

At this point, for each bidding period of task $g$, we have collected an $M \times K$ nested reward matrix, where the element $(m, k)$ corresponds to the $k$-th rollout trajectory under strategy prompt $p_m$, with reward $r_{m,k}$. The nested structure of this reward matrix encapsulates the dual contributions of the Planner's strategy selection and the Executor's bidding execution, providing a natural basis for inter-layer credit assignment via D-GRPO.
\subsubsection{Decoupled Group Relative Policy Optimization}
In recent years, Group Relative Policy Optimization (GRPO) \cite{shao2024deepseekmath} has been widely adopted for reinforcement fine-tuning of LLMs. However, directly applying GRPO to the hierarchical architecture of AIGB-R1 faces the aforementioned credit assignment problem, as the reward signal couples the joint performance of the Planner and the Executor, and directly using GRPO for advantage estimation may therefore conflate strategy quality with execution performance.
To this end, we propose Decoupled Group Relative Policy Optimization (D-GRPO), which constructs separate advantage signals for the Planner and the Executor through an orthogonal decomposition of the reward signal.

GRPO measures the relative superiority of trajectory \(k\) within a sampled group based on its reward deviation from the group mean, i.e., \(r_k- \frac{1}{K}\sum_{\ell=1}^{K}r_\ell\).
For the reward matrix produced by hierarchical rollout, we introduce the mean reward $\bar{r}_m$ of strategy prompt $p_m$ to further decompose the reward deviation of the nested reward matrix along the sampling hierarchy into two orthogonal components, namely the inter-group reward
component and the intra-group reward component:
\begin{equation}
\label{eq:reward_decomp}
    \begin{cases}
    \bar r_m=\frac{1}{K}\sum_{k}r_{m,k}\\
    \bar{r}= \frac{1}{MK}\sum_{m}\sum_{k} r_{m,k},\\
    r_{m,k} - \bar{r} = \bigl(\bar{r}_m - \bar{r}\bigr) + \bigl(r_{m,k} - \bar{r}_m\bigr)
    \end{cases}
\end{equation}
The inter-group component $\bar{r}_m - \bar{r}$ averages over $K$ rollouts for each strategy prompt $p_m$, reducing the randomness of individual Executor executions, so that this signal reflects the relative performance of the Planner's strategy under the current Executor. The intra-group component $r_{m,k} - \bar{r}_m$ subtracts the mean reward of the corresponding strategy prompt from each trajectory's reward, yielding a relative performance signal for the Executor under the same strategy prompt.
The orthogonality of this decomposition is formally proved in Appendix~\ref{app:orthogonality-proof}.

We normalize the inter-group and intra-group reward components along the strategy prompt dimension and the trajectory dimension respectively,
yielding the advantage $\hat{A}_m^{\mathcal{P}}$ for each strategy prompt $p_m$ output by the Planner and the advantage
$\hat{A}_{m,k}^{\mathcal{E}}$ for each bidding trajectory $\tau_{m,k}$ output by the Executor:
\begin{equation}                                                                                   
    \hat{A}_m^{\mathcal{P}} = \frac{\bar{r}_m - \bar{r}}{\sigma^{\mathcal{P}}}, \quad                    
    \hat{A}_{m,k}^{\mathcal{E}} = \frac{r_{m,k} - \bar{r}_m}{\sigma^{\mathcal{E}}}                   
    \label{eq:advantages}                                                                            
\end{equation} 
where $\sigma^{\mathcal{P}}$ and $\sigma^{\mathcal{E}}$ denote the standard deviations of the inter-group and intra-group reward components, respectively.

At each training step, we jointly optimize the Planner and the Executor by maximizing the following objectives:
\begin{equation}
    \small
    \begin{split}
    \mathcal{J}&_{GRPO}(\theta_{\mathcal{P}}) =
    \mathbb{E}_{(\mathcal{C},\mathcal{H})\sim\mathcal{D},\;
    \{p_m\}_{m=1}^{M}\sim\pi_{\mathrm{old}}^{\mathcal{P}}
    (p\,|\,\mathcal{C},\mathcal{H})}
    \frac{1}{M}\sum_{m=1}^{M}\frac{1}{|p_m|}\sum_{t=1}^{|p_m|}\\
    &\quad \bigg\{\min\Big(\rho_{m,t}^{\mathcal{P}}\,\hat{A}_{m}^{\mathcal{P}},\;
    \operatorname{clip}\big(\rho_{m,t}^{\mathcal{P}},\,1\!\pm\!\epsilon \big
  )\,\hat{A}_{m}^{\mathcal{P}}\Big)
    - \beta_{\mathcal{P}}\,\mathbb{D}_{\mathrm{KL}}\big[\pi_{\theta_{\mathcal{P}}}\,\|\,\pi_{\mathrm{ref}}^{\mathcal{P}}\big]
    \bigg \}
    \end{split}
    \label{eq:loss_planner}
\end{equation}

\begin{equation}
      \small
      \begin{split}
      \mathcal{J}&_{GRPO}(\theta_{\mathcal{E}}) =
      \mathbb{E}_{p_m\sim P,\;
      \{\tau_{m,k}\}_{k=1}^{K}\sim\pi_{\mathrm{old}}^{\mathcal{E}}(\tau\,|\,p_m)}
      \frac{1}{K}\sum_{k=1}^{K}\frac{1}{|\tau_{m,k}|}\sum_{t=1}^{|\tau_{m,k}|}\\
      &\quad \bigg\{\min\Big(\rho_{m,k,t}^{\mathcal{E}}\,\hat{A}_{m,k}^{\mathcal{E}},\;
      \operatorname{clip}\big(\rho_{m,k,t}^{\mathcal{E}},\,1\!\pm\!\epsilon\big)\,\hat{A}_{m,k}^{\mathcal{E}}\Big)
      - \beta_{\mathcal{E}}\,\mathbb{D}_{\mathrm{KL}}\big[\pi_{\theta_{\mathcal{E}}}\,\|\
\pi_{\mathrm{ref}}^{\mathcal{E}}\big]
      \bigg\}
      \end{split}
      \label{eq:loss_executor}
  \end{equation}
where $\rho$ denotes the importance sampling weight, and $\epsilon$ and $\beta$ are training hyperparameters.
The detailed post-training procedure is shown in Appendix \ref{app:algorithm}.

\begin{table*}[!t]
  \caption{Performance comparison. The boldface denotes the best performance.}                                           
  \vspace{-3mm}
  \setlength\tabcolsep{8pt}                                                                  
  \renewcommand{\arraystretch}{1.0}                                         
  \resizebox{\textwidth}{!}{
  \begin{tabular}{c|cccccccccccccccc}
  \hline\hline
  Dataset & USCB & CQL  & IQL  & BCQ  & DT   & DT-score & DiffBid  &CBD &
  GAS  &GAVE &PRO-Bid &GRAD  & LBM  & AIGB-R1   \\ \hline
  \multirow{1}{*}{AuctionNet}   & 157  & 171  & 281  & 270  & 329  & 334  & 152 & 298
   & 359  & 376 & 372 & 372 &348 & $\mathbf{385}\pm 3.61$   \\ \hline
   \multirow{1}{*}{AuctionNet-sparse}  & 17.5  & 22.2  & 30.0  & 30.5  & 29.6  & 33.2  & 19.5
   &37.0 & 36.1  & 37.2 &38.1 &37.4  & 33.4 & $\mathbf{39.0}\pm 0.37$  \\ \hline\hline 
  \end{tabular}}                                                                        
  \label{tab:comp_all}                                                                                 
\end{table*}    


\section{Experiment}
\label{sec:exp}
\subsection{Experimental Setup}
\subsubsection{Dataset}
To comprehensively evaluate the performance of AIGB-R1 on large-scale advertising auctions, we conduct experiments on AuctionNet, a large-scale public real-world bidding dataset released by Alibaba, together with its more challenging sparse variant, AuctionNet-sparse. Details are provided in Appendix \ref{app:dataset}.

\subsubsection{Evaluation Metrics}
We adopt the following metrics to evaluate the performance:
\begin{itemize}[leftmargin=*, noitemsep]
  \item \textbf{Conversions}: The total conversions obtained during the bidding period, calculated as $\sum_{i} o_i v_i$.
  \item \textbf{Score}: A metric for jointly evaluating conversions and CPA constraint satisfaction, defined as $\text{score} = (\sum_i o_i
v_i) \times \text{penalty}$, where the penalty term $\text{penalty} = \min\left\{\left(\frac{C}{C_{real}}\right)^2, 1\right\}$, and $C_{real}$
and $C$ denote the realized CPA and CPA constraint, respectively.
\end{itemize}

\subsubsection{Baselines}
Following the AuctionNet benchmark's standard evaluation protocol, we compare AIGB-R1 against various state-of-the-art baselines on the test sets, with its experience buffer reset before evaluation.
For RL-based methods, we compare with the online RL method USCB \cite{he2021unified} and offline RL methods BCQ \cite{fujimoto2019off}, CQL \cite{kumar2020conservative}, and IQL \cite{kostrikov2021offline}. 
For generative methods, we compare with both diffusion-based and DT-based methods.
Diffusion-based baselines include DiffBid ~\cite{guo2024generative}, which generates bidding trajectories via conditional diffusion modeling, and CBD ~\cite{li2025generative}, which adopts a diffusion completer-aligner framework for large-scale competitive auctions.
DT-based baselines include DT ~\cite{chen2021decision} and DT-score ~\cite{xu2024auto}, which leverages a reward function to jointly represent the winning value and KPI constraints.
We further include GAS ~\cite{li2025gas}, which performs post-training search via MCTS, GAVE ~\cite{gao2025generative}, which employs value-guided exploration, PRO-Bid ~\cite{wu2026constraint}, which introduces Pareto-prioritized regret optimization for constraint-aware bidding, and GRAD ~\cite{lei2026generative}, which develops a large-scale pre-trained DT-based bidding model.
For LLM-based methods, we compare with LBM~ \cite{li2026lbm}, which designs a hierarchical Think-Act architecture with GQPO-based fine-tuning.

\subsubsection{Implementation Details}
We deploy the baseline models following their original implementations, with further hyperparameter tuning to optimize performance. All baseline methods are trained on the AuctionNet datasets. The Planner module uses the Qwen2.5-32B-Instruct model with parameter-efficient fine-tuning based on the ms-swift framework \cite{zhao2024swiftascalablelightweightinfrastructure}, with a batch size of 128 and a learning rate of 5e-6. The Executor module follows the official code of DT \cite{chen2021decision} and Prompt-DT \cite{xu2022prompting}
with necessary adaptations. The Executor adopts a causal transformer architecture with 6 attention layers and 8 attention heads, with a hidden size of 512, constituting a lightweight decision model. The Executor is optimized using AdamW with a learning rate of 1e-5 and a batch size of 512. Our training is conducted on 8 NVIDIA A100 GPUs using the PyTorch framework. We pre-train the PDT Executor for 250k steps, followed by 300 steps of joint post-training of the Planner and Executor. We implement the D-GRPO algorithm compatible with the PDT architecture following GAM \cite{dou2026gam}, using $M=3$ and $K=8$. We set the reward penalty exponent to $\eta=2$ during post-training, consistent with the evaluation metric. The reported results are averaged over five random runs. We report AIGB-R1's inference latency and discuss its computational efficiency in Appendix \ref{efficiency}.

\subsection{Performance Comparison}

We present a comprehensive comparison between AIGB-R1 and various baseline approaches, using Score as the evaluation metric, with results on AuctionNet and AuctionNet-sparse summarized in Table~\ref{tab:comp_all}. To further assess the generalization ability and robustness of our method, we evaluate AIGB-R1 against state-of-the-art generative baselines under different budget settings. 
Specifically, we train AIGB-R1 under the standard budget setting and vary only the test-time budget across five ratios: \{50\%, 75\%, 100\%, 125\%, 150\%\}, with results reported in Table~\ref{tab:budget_ratio}.
Our results show that:

\begin{itemize}[leftmargin=*, topsep=1pt, itemsep=0pt, parsep=0pt, partopsep=0pt]
  \item DT-based generative bidding methods significantly outperform RL methods such as IQL, CQL, and USCB, demonstrating the effectiveness of
generative models in solving complex sequential decision-making tasks. Notably, advanced generative baselines such as GAS, GAVE, PRO-Bid, and GRAD further outperform DT and its simple variants through Monte Carlo tree search, value-guided exploration, constraint-aware regret optimization, and large-scale pre-training, respectively, indicating the effectiveness of incorporating specialized strategy exploration and optimization mechanisms into generative auto-bidding.
  \item Although the LLM-based method outperforms the standard DT, it still falls significantly behind advanced generative baselines, likely due to the inherent limitations of LLMs as decision models in handling complex numerical decision-making tasks. In contrast, AIGB-R1 leverages LLMs for high-level strategy planning while training a dedicated PDT as the decision model, fully exploiting LLMs' reasoning and planning capabilities while circumventing this limitation.
  \item AIGB-R1 demonstrates the best performance across  all experimental settings, consistently outperforming existing state-of-the-art baselines. These results validate the effectiveness of our hierarchical Planner-Executor architecture and end-to-end self-evolving post-training paradigm.
\end{itemize}
\begingroup
\setlength{\intextsep}{6pt}
\begin{table}[H]
\centering
\caption{Generalization to various budget settings against advanced baselines.}
\vspace{-3mm}
\label{tab:budget_ratio}
\setlength{\tabcolsep}{3pt}
\renewcommand{\arraystretch}{0.90}

\begin{tabularx}{\columnwidth}{
  >{\centering\arraybackslash}p{0.18\columnwidth}|
  *{5}{>{\centering\arraybackslash}X}
}
\hline\hline
\multicolumn{6}{c}{AuctionNet} \\
\hline
Method & 50\% & 75\% & 100\% & 125\% & 150\% \\
\hline
DiffBid & 54  & 100 & 152 & 193 & 234 \\
GAS     & 193 & 287 & 359 & 409 & 461 \\
GAVE    & 201 & 296 & 376 & 421 & 467 \\
PRO-Bid & 204 & 291 & 372 & 426 & 471 \\
GRAD    & 204 & 293 & 372 & 432 & 476 \\
AIGB-R1 & \textbf{231}
        & \textbf{313}
        & \textbf{385}
        & \textbf{437}
        & \textbf{489} \\
\hline\hline
\multicolumn{6}{c}{AuctionNet-sparse} \\
\hline
Method & 50\% & 75\% & 100\% & 125\% & 150\% \\
\hline
DiffBid & 9.87 & 15.4 & 19.5 & 25.3 & 30.8 \\
GAS     & 18.4 & 27.5 & 36.1 & 40.0 & 46.5 \\
GAVE    & 19.6 & 28.3 & 37.2 & 42.7 & 47.4 \\
PRO-Bid & 21.0 & 28.8 & 38.1 & 43.7 & 49.9 \\
GRAD    & 20.0 & 28.5 & 37.4 & 43.2 & 47.5 \\
AIGB-R1 & \textbf{21.1}
        & \textbf{29.6}
        & \textbf{39.0}
        & \textbf{45.1}
        & \textbf{50.9} \\
\hline\hline
\end{tabularx}
\end{table}

\endgroup

\subsection{Ablation Study}
We conduct an ablation study by evaluating the following modified variants of AIGB-R1:

\begin{itemize}[leftmargin=*, topsep=1pt, itemsep=0pt, parsep=0pt, partopsep=0pt]
  \item \textbf{AIGB-R1(E) w/o RL}: Uses the pre-trained Executor without the Planner or post-training stage.
  \item \textbf{AIGB-R1(E)}: Removes the Planner module and applies standard GRPO to post-train the pre-trained Executor.
  \item \textbf{AIGB-R1(SP) w/o RL}: Removes the post-training stage and uses a frozen LLM as a static Planner to provide strategy prompts for
the pre-trained Executor at test time.
  \item \textbf{AIGB-R1(SP)}: Uses a frozen static Planner during post-training, i.e., only the Executor's parameters are updated.
  \item \textbf{AIGB-R1(G)}: Uses standard GRPO instead of D-GRPO to jointly post-train the Planner and Executor under otherwise identical post-training settings.
  \item \textbf{AIGB-R1(M)}: Removes the Planner's experience buffer, generating strategy prompts solely from advertiser contexts under otherwise identical post-training settings.
\end{itemize}
Table~\ref{tab:ablation} compares the variants, showing that each core component of AIGB-R1 contributes to overall performance.
AIGB-R1(E) w/o RL performs the worst, reflecting the inherent limitation that purely offline pre-trained generative models are
constrained by the quality and coverage of training datasets.
AIGB-R1(E) and AIGB-R1(SP) significantly outperform their respective w/o RL variants, confirming the importance of post-training for improving generative models' performance.
Meanwhile, the AIGB-R1(SP) series consistently outperforms the AIGB-R1(E) series, indicating that LLMs' reasoning and planning capabilities
provide effective macro-level strategy priors for generative decision models, validating the rationale for our hierarchical architecture that
decouples planning and execution. We also provide an instruction-following evaluation and an example Planner output in Appendix \ref{app:instruction_following}.
Furthermore, AIGB-R1 outperforms AIGB-R1(SP), demonstrating the importance of joint post-training for the hierarchical architecture. 
AIGB-R1(G), which also uses joint post-training, significantly underperforms AIGB-R1, further supporting the benefit of D-GRPO's decoupled advantage estimation.
Moreover, AIGB-R1 outperforms AIGB-R1(M) despite slightly fewer conversions, indicating that experience feedback helps guide strategy exploration and optimization, thereby contributing to strategy self-evolution.

\begingroup
\setlength{\intextsep}{6pt}
\begin{table}[H]
\centering
\caption{Ablation Study}
\vspace{-3mm}
\label{tab:ablation}
\setlength{\tabcolsep}{8pt}
\renewcommand{\arraystretch}{0.90}

\begin{tabularx}{\columnwidth}{
  c|
  >{\centering\arraybackslash}X
  >{\centering\arraybackslash}X
}
\hline\hline
Model & Score & Conversions \\
\hline
AIGB-R1(E) w/o RL  & 343 & 362 \\
AIGB-R1(E)         & 367 & 395 \\
AIGB-R1(SP) w/o RL & 351 & 388 \\
AIGB-R1(SP)        & 371 & 396 \\
AIGB-R1(G)      & 348 & 393 \\
AIGB-R1(M)      & 378 & \textbf{406} \\
\hline
AIGB-R1 & \textbf{385} & 404 \\
\hline\hline
\end{tabularx}
\end{table}

\endgroup

\section{Conclusion}
\label{sec:con}
In this work, we propose a hierarchical self-evolving auto-bidding framework, termed AIGB-R1, which achieves autonomous exploration and continuous evolution of bidding strategies through an experience-driven self-evolving loop powered by LLMs' reasoning, planning, and reflection capabilities.
We design a hierarchical Planner-Executor architecture that employs LLMs for high-level bidding strategy planning and PDT for low-level execution, fully leveraging LLMs' reasoning and planning capabilities while circumventing their inherent limitations in complex numerical decision-making tasks.
We adopt a two-stage training pipeline comprising offline pre-training and post-training alignment, and propose an end-to-end self-evolving post-training paradigm for efficient strategy exploration and continuous self-improvement.
We also design a hierarchical rollout mechanism and the corresponding D-GRPO algorithm to achieve end-to-end optimization via advantage decoupling.
Experimental results on a large-scale public dataset demonstrate AIGB-R1's effectiveness, providing a promising solution for LLM-empowered auto-bidding.

\newpage
\balance
\bibliographystyle{ACM-Reference-Format}
\bibliography{software}

\newpage
\appendix
\clearpage
\ifcsname nobalance\endcsname\nobalance\fi
\ifcsname @ACM@balancefalse\endcsname
  \csname @ACM@balancefalse\endcsname
\fi
\begingroup
\raggedbottom
\setlength{\emergencystretch}{2em}
\setlength{\intextsep}{4pt}
\setlength{\textfloatsep}{6pt}
\setlength{\floatsep}{6pt}
\setlength{\abovedisplayskip}{4pt}
\setlength{\belowdisplayskip}{4pt}
\setlength{\abovedisplayshortskip}{2pt}
\setlength{\belowdisplayshortskip}{2pt}
\captionsetup{skip=2pt}

\twocolumn[{%
\section{Appendix}
\begin{minipage}{\textwidth}
    \captionof{table}{Distributional consistency of AuctionNetEnv against the AuctionNet dataset.}
    \label{tab:traffic_fidelity}
    \centering
    \setlength\tabcolsep{6pt}
    \renewcommand{\arraystretch}{0.90}
    \resizebox{\textwidth}{!}{
    \begin{tabular}{c|c|ccc|ccc|ccc}
    \hline\hline
    \multirow{2}{*}{Dataset}
      & \multirow{2}{*}{Metric}
      & \multicolumn{3}{c|}{Mean}
      & \multicolumn{3}{c|}{Std}
      & \multicolumn{3}{c}{Median} \\
    \cline{3-11}
      & & Data & Sim & Ratio
        & Data & Sim & Ratio
        & Data & Sim & Ratio \\
    \hline
    \multirow{3}{*}{AuctionNet}
      & pValue
        & 0.00479 & 0.00480 & 1.003
        & 0.00377 & 0.00364 & 0.966
        & 0.00385 & 0.00388 & 1.008 \\
      & LWC
        & 0.0892 & 0.0847 & 0.950
        & 0.0293 & 0.0243 & 0.830
        & 0.0810 & 0.0785 & 0.968 \\
      & PV count/tick
        & 10416 & 10416 & 1.000
        & 6353 & 6743 & 1.061
        & 10573 & 10483 & 0.991 \\
    \hline
    \multirow{3}{*}{\shortstack{AuctionNet-\\sparse}}
      & pValue
        & 0.000479 & 0.000480 & 1.002
        & 0.000377 & 0.000364 & 0.966
        & 0.000385 & 0.000388 & 1.008 \\
      & LWC
        & 0.0924 & 0.0883 & 0.956
        & 0.0273 & 0.0213 & 0.780
        & 0.0851 & 0.0835 & 0.980 \\
      & PV count/tick
        & 10416 & 10416 & 1.000
        & 6353 & 6743 & 1.061
        & 10573 & 10483 & 0.991 \\
    \hline\hline
    \end{tabular}
    }

\end{minipage}
\vspace{5pt}
}]

\subsection{AuctionNetEnv Details}
\label{app:AuctionNetEnv}
\subsubsection{Construction Details}

Based on the public, large-scale auto-bidding decision benchmark AuctionNet \cite{su2024auctionnet}, we construct an interactive bidding simulation environment, AuctionNetEnv, for flexible agent-environment interaction during post-training. AuctionNetEnv adapts the ad opportunity generation module, ad auction module, and various auto-bidding opponent agents to match the auction mechanism and distributional characteristics of the AuctionNet dataset, and wraps them in an OpenAI Gym-style interface for rollouts.

In AuctionNetEnv, each rollout trajectory corresponds to a complete bidding period of $T=48$ timesteps. In each bidding period, the ad opportunity generator produces about 500K impression opportunities in total, distributed across the 48 timesteps following the temporal traffic patterns of a real-world advertising platform. The task pool $\mathcal{G}$ contains 48 bidding tasks, each with a distinct advertiser configuration including budget, CPA constraint, and advertising category. During post-training, an advertiser configuration is sampled from $\mathcal{G}$ for the player agent, which is controlled by AIGB-R1 and engages in online bidding against opponent agents pre-trained on the training dataset using various algorithms (PID, Online LP, CQL, IQL, etc.), simulating the heterogeneous competitive dynamics of real-world advertising platforms. The opponent agents and their configurations follow those provided by AuctionNet.
At each timestep, the player agent bids for each impression opportunity; if the bid exceeds the highest bid among all opponent agents, it wins the impression and pays the winning cost. For each won impression,  the conversion outcome is generated by Bernoulli sampling with the impression value as the success probability.

\subsubsection{Distributional Consistency of AuctionNetEnv}
To assess whether AuctionNetEnv can reproduce the key distributional characteristics of the AuctionNet dataset, we compare its simulated traffic samples with the corresponding test-set samples. Specifically, we consider three key metrics: user conversion probability (pValue), least winning cost (LWC, i.e., the highest bid among competitors), and traffic volume per tick (PV count/tick). For each metric, we report the mean, standard deviation, and median, and calculate the ratio of the simulated statistic to its test-set counterpart (Sim/Data), as summarized in Table~\ref{tab:traffic_fidelity}. A ratio closer to 1.0 indicates greater distributional consistency between the simulated samples and the AuctionNet test-set samples.

\subsubsection{Fidelity of AuctionNetEnv}
Beyond the marginal distributional comparison in Table~\ref{tab:traffic_fidelity}, we assess AuctionNetEnv fidelity through trajectory replay on the test sets of AuctionNet and
AuctionNet-sparse. We replay complete bidding trajectories in AuctionNetEnv and compare the simulated pValue, conversion, and LWC with the corresponding ground-truth values in the test sets at each timestep. Table~\ref{tab:trajectory_replay} reports MAE and RMSE to quantify numerical discrepancies, and Pearson's correlation coefficient to assess linear association.

\begin{table}[H]
    \caption{Fidelity of AuctionNetEnv.}
    \label{tab:trajectory_replay}
    \centering
    \setlength\tabcolsep{3pt}
    \renewcommand{\arraystretch}{0.90}
    \resizebox{\columnwidth}{!}{
    \begin{tabular}{c|c|ccc}
    \hline\hline
    Dataset & Metric & MAE & RMSE & Correlation \\
    \hline
    \multirow{3}{*}{AuctionNet}
      & pValue     & 0.00007  & 0.00009  & 1.0   \\
      & Conversion & 1.724    & 4.067    & 0.912 \\
      & LWC        & 0.0097   & 0.0131   & 0.924 \\
    \hline
    \multirow{3}{*}{\shortstack{AuctionNet-\\sparse}}
      & pValue     & 0.000007 & 0.000009 & 1.0   \\
      & Conversion & 0.314    & 0.79     & 0.901 \\
      & LWC        & 0.01     & 0.0148   & 0.891 \\
    \hline\hline
    \end{tabular}
    }
\end{table}
\subsection{Proof of Reward Decomposition Orthogonality}
\label{app:orthogonality-proof}

\begin{proposition}
\label{prop:orthogonality}
The centered reward deviation of the nested reward matrix can be decomposed into two orthogonal components.
\end{proposition}

\begin{proof}
\setlength{\abovedisplayskip}{2pt}
\setlength{\belowdisplayskip}{2pt}
\setlength{\abovedisplayshortskip}{1pt}
\setlength{\belowdisplayshortskip}{1pt}
Consider the hierarchical rollout reward matrix:
\begin{equation}
\small
R = \{r_{m,k}\}_{m=1,\,k=1}^{M,\,K},
\end{equation}
where $m$ indexes the $m$-th strategy prompt and $k$ indexes the $k$-th rollout trajectory under this prompt. 

Define the mean reward under the $m$-th strategy prompt as
{\small $\bar{r}_m = \frac{1}{K}\sum_{k=1}^{K} r_{m,k}$}
and the global mean reward as
{\small $\bar{r} = \frac{1}{MK}\sum_{m=1}^{M}\sum_{k=1}^{K} r_{m,k}$}.
Then the centered reward deviation of each element can be decomposed as
\begin{equation}
\small
r_{m,k} - \bar{r} = \underbrace{(\bar{r}_m - \bar{r})}_{B_{m,k}} + \underbrace{(r_{m,k} - \bar{r}_m)}_{W_{m,k}},
\label{eq:reward-decomposition}
\end{equation}
where $B$ and $W$ denote the inter-group and intra-group components, respectively. In matrix form,
\begin{equation}
\small
R - \bar{r}\,\mathbf{1}\mathbf{1}^{\!\top} = B + W.
\end{equation}

By the definitions of $B$ and $W$, their Frobenius inner product is
\begin{equation}
\small
\langle B,\, W \rangle_F
= \sum_{m=1}^{M}\sum_{k=1}^{K} (\bar{r}_m - \bar{r})\,(r_{m,k} - \bar{r}_m).
\end{equation}
For a fixed $m$, the factor $(\bar{r}_m - \bar{r})$ is independent of $k$, so
\begin{equation}
\small
\langle B,\, W \rangle_F
= \sum_{m=1}^{M} (\bar{r}_m - \bar{r}) \sum_{k=1}^{K} (r_{m,k} - \bar{r}_m).
\end{equation}
By the definition of $\bar{r}_m$,
\begin{equation}
\small
\sum_{k=1}^{K}(r_{m,k} - \bar{r}_m)
= \sum_{k=1}^{K} r_{m,k} - K\bar{r}_m
= K\bar{r}_m - K\bar{r}_m = 0.
\end{equation}
Therefore, $\langle B,\, W \rangle_F = 0$, proving that the inter-group and intra-group components are orthogonal.

\end{proof} 

\subsection{Post-training Procedure}
\label{app:algorithm}
Algorithm~\ref{alg:dgrpo} describes the complete post-training procedure of AIGB-R1. 
At each training step, we first sample a bidding task and obtain its contextual information and historical bidding experience. 
We then perform the experience-guided hierarchical rollout. The LLM Planner then generates $M$ diverse strategy prompts based on the context and historical experience, and the PDT Executor performs $K$ rollouts in AuctionNetEnv for each
strategy prompt. 
During advantage estimation, we decouple the reward signal into inter-group and intra-group components to construct separate advantage signals for the Planner and the Executor respectively. 
Finally, we update the parameters of the Planner and the Executor using their corresponding advantages, and append the rollout results to the experience buffer.

\begingroup
\SetAlFnt{\fontsize{8}{8.8}\selectfont}

\par\noindent
\begin{minipage}{\columnwidth}
\csname @twocolumnfalse\endcsname
\begin{algorithm}[H]
\SetAlgoNoLine
\caption{Post-training Procedure}
\label{alg:dgrpo}
{\bf Input:} Planner $\pi_{\theta_{\mathcal{P}}}$,
pre-trained Executor $\pi_{\theta_{\mathcal{E}}}$,
AuctionNetEnv task pool $\mathcal{G}$,
$M$ prompts and $K$ rollouts per prompt\;
Initialize buffers $\mathcal{H}_g \leftarrow \emptyset$
for all $g \in \mathcal{G}$\;
Set reference models
$(\pi_{\mathrm{ref}}^{\mathcal{P}},
  \pi_{\mathrm{ref}}^{\mathcal{E}})
\leftarrow
(\pi_{\theta_{\mathcal{P}}},
 \pi_{\theta_{\mathcal{E}}})$\;
\For{each training step}{
  Sample $g \in \mathcal{G}$ and obtain
  $\mathcal{C}_g, \mathcal{H}_g$\;
  Generate $M$ diverse prompts
  $\{p_m\}_{m=1}^{M}
  \sim \pi_{\theta_{\mathcal{P}}}
  (\cdot \mid \mathcal{C}_g, \mathcal{H}_g)$\;
  \For{$(m,k) \in
    \{1,\ldots,M\} \times \{1,\ldots,K\}$}{
    Roll out
    $\tau_{m,k} \sim
    \pi_{\theta_{\mathcal{E}}}(\cdot \mid p_m)$
    in AuctionNetEnv; compute $r_{m,k}$\;
  }
  $\bar{r}_m \leftarrow
    \frac{1}{K}\sum_{k=1}^{K} r_{m,k}$,
  \quad
  $\bar{r} \leftarrow
    \frac{1}{MK}\sum_{m,k} r_{m,k}$\;
  $\hat{A}_m^{\mathcal{P}} \leftarrow
    \frac{\bar{r}_m-\bar{r}}{\sigma^{\mathcal{P}}}$
    \tcp*{Inter-group}
  $\hat{A}_{m,k}^{\mathcal{E}} \leftarrow
    \frac{r_{m,k}-\bar{r}_m}{\sigma^{\mathcal{E}}}$
    \tcp*{Intra-group}
  Update $\theta_{\mathcal{P}}, \theta_{\mathcal{E}}$
  by maximizing their respective GRPO objectives\;
  Update $\mathcal{H}_g$ with rollout results\;
  Update the best-strategy anchor if
  $\max_m \bar{r}_m$ improves the current best\;
}
{\bf Output:} Optimized
$\pi_{\theta_{\mathcal{P}}}$ and
$\pi_{\theta_{\mathcal{E}}}$\;
\end{algorithm}
\end{minipage}
\par
\endgroup

\subsection{Instruction Following}
\label{app:instruction_following}
To assess whether the Executor effectively follows the Planner's strategy, we conduct an instruction-following experiment in which the Planner is instructed to generate more aggressive or more conservative bidding strategy prompts during evaluation. The results in Table~\ref{tab:instruction_following} demonstrate effective instruction following, with aggressive prompts yielding higher budget utilization, CPA ratio, and conversions than conservative prompts.
For a detailed illustration of the strategy prompt design in AIGB-R1, we present an example of the Planner's actual output in Figure~\ref{fig:planner_output}.

\begin{table}[H]
    \caption{Instruction-following experiment with different bidding strategy prompts. Base denotes the default setting.}
    \label{tab:instruction_following}
    \centering
    \setlength\tabcolsep{3pt}
    \renewcommand{\arraystretch}{1.0}
    \resizebox{\columnwidth}{!}{
    \begin{tabular}{c|cccc}
    \hline\hline
    Instruction
      & \shortstack{Budget Utilization $\uparrow$}
      & CPA Ratio $\downarrow$
      & Conversions $\uparrow$
      & Score $\uparrow$ \\
    \hline
    Aggressive   & 0.975 & 0.940 & 400 & 345 \\
    Conservative & 0.804 & 0.816 & 382 & 358 \\
    Base         & 0.918 & 0.847 & 404 & 385 \\
    \hline\hline
    \end{tabular}
    }
\end{table}

\begin{figure}[H]
    \centering
    \captionsetup{skip=2pt}
    \begin{tcolorbox}[
        colback=blue!3,
        colframe=blue!25,
        boxrule=0.5pt,
        boxsep=0pt,
        left=2pt,
        right=2pt,
        top=1pt,
        bottom=1pt,
        before skip=0pt,
        after skip=0pt
    ]
    \begin{lstlisting}[
        basicstyle=\ttfamily\footnotesize,
        breaklines=true,
        breakatwhitespace=true,
        columns=fullflexible,
        keepspaces=true,
        showstringspaces=false,
        aboveskip=0pt,
        belowskip=0pt
    ]
{"risk_awareness": "AGGRESSIVE",
 "pacing_strategy": "FRONT_LOADED",
 "volatility": "VOLATILE",
 "reasoning": "Prioritize conversion volume with more assertive bids and allocate more budget early to capture valuable opportunities. Allow larger bid adjustments to respond quickly to changing market conditions."}
    \end{lstlisting}
    \end{tcolorbox}
    \caption{Illustration of the Planner's strategy prompt.}
    \label{fig:planner_output}
\end{figure}

\subsection{Computational Efficiency}
\label{efficiency}
We measure all modules' inference latency on an A100 GPU with vLLM acceleration, as shown in Table~\ref{tab:inference_latency}. AIGB-R1's hierarchical design separates strategy planning and bidding execution across timescales, avoiding delays from invoking the Planner for every bid request. The Planner runs only once before each bidding period, usually one day, to generate a fixed bidding strategy reused
throughout the period. For each bid request during the period, only the Executor runs to respond in real time.
Table~\ref{tab:inference_latency} shows that the Executor can meet the millisecond-level latency needs of real-time bidding. Thus, AIGB-R1
effectively separates complex, time-consuming macro-level strategy reasoning from frequent, real-time bidding executions, supporting the
computational feasibility of deployment under industrial-scale traffic.
\begin{table}[H]
    \caption{Inference latency.}
    \label{tab:inference_latency}
    \centering
    \setlength\tabcolsep{6pt}
    \renewcommand{\arraystretch}{0.90}
    \begin{tabularx}{\columnwidth}{
        >{\centering\arraybackslash}X|
        >{\centering\arraybackslash}X
    }
    \hline\hline
    Module & Inference Latency \\
    \hline
    AIGB-R1 Planner  & 5.2\,s \\
    AIGB-R1 Executor & 5.3\,ms \\
    \hline\hline
    \end{tabularx}
\end{table}

\subsection{Dataset Details}
\label{app:dataset}
The AuctionNet benchmark comprises two datasets, AuctionNet and AuctionNet-sparse, where AuctionNet-sparse is a sparse variant of AuctionNet with fewer conversions. Each dataset contains 21 advertising delivery periods, each with approximately 500,000 impression opportunities, divided into 48 time intervals. Detailed parameters are summarized in Table~\ref{tab:dataset}.
\begin{table}[H]
  \centering
  \caption{Parameters of AuctionNet and AuctionNet-sparse.}
  \renewcommand{\arraystretch}{0.85}
  \scalebox{0.9}{
    \begin{tabular}{ccc}
    \toprule
    Parameters & AuctionNet & AuctionNet-sparse \\
    \midrule
    Trajectories & 479,376 & 479,376 \\
    Timesteps in a trajectory & 48 & 48 \\
    State dimension & 16 & 16 \\
    Action dimension & 1 & 1 \\
    Action range & [0, 493] & [0, 589] \\
    Impression's value range & [0, 1] & [0, 1] \\
    CPA range & [6, 12] & [60, 130] \\
    Total conversion range & [0, 1512] & [0, 57] \\
    \bottomrule
    \end{tabular}%
  }
  \label{tab:dataset}
\end{table}

\endgroup

\end{document}